\documentclass[final,5p,10pt,twocolumn]{elsarticle}

\usepackage{amssymb}
\usepackage{amsmath}
\usepackage{graphicx}
\usepackage{amsmath,amsthm}
\usepackage{amssymb,empheq}
\usepackage{latexsym}
\usepackage{multirow}

\graphicspath{{./}{./Fig/}{./Figs/}}

\usepackage{algorithm}
\usepackage{algorithmic}
\usepackage{multirow,url,hyperref}

\biboptions{comma,square,sort}

\newcommand{\BA}{{\mathbf{A}}}

\newcommand{\Ba}{{\mathbf{a}}}
\newcommand{\By}{{\mathbf{y}}}

\newcommand{\Bepsilon}{{\mathbf{\epsilon}}}

\newcommand{\T}{{\!\top}}

\newcommand{\Bbeta}{{\boldsymbol{\beta}}}

\newcommand{\st}{{\,\,\mathrm{s.t.\,\,}}}

\newcommand{\comment}[1]{\textsl{\Large #1}}

\renewcommand{\comment}[1]{}

\newtheorem{thm}{Theorem}
\begin{document}

\begin{frontmatter}




\title{
          Fast Approximate $L_{\infty}$ Minimization: Speeding Up Robust Regression
      }

\author{Fumin Shen$^{2,1}$}
\author{Chunhua Shen$^1$}
\author{Rhys Hill$^1$}
\author{Anton van den Hengel$^1$}
\author{Zhenmin Tang$^2$}
\address{$^1$ School of Computer Science, 
The University of Adelaide, Australia
\\
$^2$ School of Computer Science and Technology, Nanjing University of Science and Technology,  China}

\begin{abstract}

    Minimization of the $L_\infty$ norm, which can be viewed as
    approximately solving the non-convex least median estimation
    problem, is a powerful method
    for outlier removal and hence robust regression. 
    However, current
    techniques for solving the problem at the heart of 
    $L_\infty$ norm minimization are slow, and therefore cannot scale
    to large problems. A new method for the 
    minimization of the $L_\infty$ norm is presented here, which provides a speedup of
    multiple orders of magnitude for data with high dimension. This
    method, termed \textsl{Fast $L_\infty$ Minimization}, allows robust
    regression to be applied to a class of problems which were
    previously inaccessible.  It is shown how the $L_\infty$
    norm minimization problem can be broken up into smaller
    sub-problems, which can then be solved extremely efficiently.
    Experimental results demonstrate the radical reduction in
    computation time, along with robustness against large numbers of
    outliers in a few model-fitting problems.

\end{abstract}


\end{frontmatter}

\section{Introduction}

    Linear least-squares (LS) estimation, or linear $L_2$ norm
    minimization (denoted as $L_2$ for brevity throughout this paper),
    is widely used in computer vision and image analysis due to its
    simplicity and efficiency. Recently the $L_2$ norm technique has
    been applied to recognition problems such as face recognition
    \cite{LRC10,ShiEHS11}.  All of these methods are
    linear regression-based and the regression residual is utilized to
    make the final classification. However, a small number of outliers
    can drastically bias $L_2$, leading to low quality estimates. Clearly, robust
    regression techniques are critical when outliers are present.

    The literature contains a range of different approaches to robust
    regression. One commonly used method is the M-estimator framework
    \cite{M_Estimators1973,Huber1981}, where the Huber function is
    minimized, rather than the conventional $L_2$ norm.  Related
    methods include L-estimators \cite{LEstimator1987} and
    R-estimators \cite{REstimator1971}. One drawback of these methods
    is that they are still vulnerable to bad leverage outliers
    \cite{Rousseeuw2008}.
    By bad leverage points, we mean those observations who are outlying in $x$-space and do not follow the
    linear pattern of the majority.  
%
%
    Least median of squares (LMS) \cite{LMS1984}, least trimmed
    squares (LTS) \cite{Rousseeuw:1987:RRO:40031} and the technique
    using data partition and M-estimation \cite{park2012robust} have
    high-breakdown points.  Although each of these regression methods
    is, in general, more robust than $L_2$, they have rarely been
    applied to object recognition problems in computer vision due to
    their computational expense \cite{RLRC2012,Fidler2006}.

Another class of robust methods have been developed to remove these abnormal observations
from the measurement data. One of the most popular methods is RANSAC \cite{RANSAC} which
attempts to maximise the size of the {\em consensus set}. RANSAC relies on iterative random
sampling and consensus testing where the size of each sample is determined by the minimum
number of data points require to compute a single solution. RANSAC's efficiency is therefore
directly tied to the time and number of data points required to compute a solution. For
example, RANSAC has been successfully applied to multiview structure-from-motion and homography
estimation problems. However, it is unclear how to apply RANSAC to visual recognition problems,
e.g., face recognition, where face images are usually in a high-dimensional space.

Sim and Hartley \cite{Sim2006} proposed an outlier-removing method using the $L_\infty$ norm,
which iteratively fits a model to the data and removes the measurement with the largest
residual at each iteration. Generally, the iterative method can fail for the $L_2$ optimization
problems, however it is valid for a wide class of $L_\infty$ problems. Sim and Hartley proved
that the set of measurements with largest residual must contain at least one outlier. Hence
continuing to iterate eventually removes all the outliers. This method is shown to be
effective in outlier detection for multiview geometry problems \cite{Sim2006,Olsson2010}.

$L_\infty$ norm minimization can be time-consuming, since at each step one
needs to solve an
 optimization problem via Second-Order Cone Programming (SOCP) or
 Linear Programming (LP)  in the
 application of multi-view geometry \cite{Kahl2005,Ke2005}.
The software package SeDuMi \cite{sedumi99} provides solvers
for both SOCP and LP problems.

In this paper, we propose a fast algorithm to minimize the  $L_\infty$
norm  for approximating the least median estimation (denoted
as $L_\infty$ for brevity throughout the paper). Observing that the $L_\infty$ norm is
determined by only a few measurements, the optimization strategy \textit{column generation} (CG)
\cite{Lubbecke05} can be applied to reduce the main problem into a set of much smaller
sub-problems. Each sub-problem can be formulated as a Quadratic
Programming (QP) problem. Due to its relatively small size, the QP
problem can be solved extremely efficiently using customized solvers. 
In particular, 
we can generate solvers using the technique introduced by Mattingley and  Boyd  \cite{CVXGEN}.
This reduction results in a speedup of several orders of magnitude for high dimensional data.

This degree of speedup allows $L_\infty$ to be applied to problems which were previously
inaccessible. We show how the $L_\infty$ outlier removal technique can be applied to several
classification problems in computer vision. Representations of objects in this type of problems
are often derived by using $L_2$ to solve equations containing samples in the same class
(or other collaborative classes) \cite{LRC10,ShiEHS11}. Representation errors are then taken
into classification where the query object is assigned to the  class corresponding to the minimal residual.
\comment{previous sentence is not clear to me. RH} 
This method is shown
to be effective on data without occlusion or outliers. However, in
real-world applications,
measurement data are almost always contaminated by noises or outliers. Before a robust
representation can be obtained via linear estimation, outlier removal is necessary. The
proposed method is shown to significantly improve the classification accuracies in our
experiments for face recognition and iris recognition on several
public datasets.

\section{Related work}

Hartley and Schaffalitzky \cite{Hartley04b} seek a globally optimal solution for multi-view
geometry problems via $L_\infty $ norm optimization, based on the fact that many geometry
problems in computer vision have a single local, and hence global, minimum under the $L_\infty$ norm.
In contrast, the commonly used $L_2$ cost function typically has multiple local minimum \cite{Hartley04b,Sim2006}. 
This work has been extended by several authors, yielding a large set of geometry problems
whose globally optimal solution can be found using the $L_\infty$ norm (Olsson provides a summary \cite{Olsson07}).

It was observed that these geometry problems are examples of \textit{quasiconvex} optimization
problems, which are typically solved by a sequence of  SOCPs using a bisection (binary search)
algorithm \cite{Kahl2005,Ke2005}. Olsson \textit{et al.} \cite{Olsson07} show that the
functions involved in the $L_\infty$ norm problems are in fact \textit{pseudoconvex}
which is a stronger condition than quasiconvex. As a consequence, several fast algorithms
have been proposed \cite{Olsson07,Hongdong09}.

Sim and Hartley~\cite{Sim2006} propose a simple method based on the $L_\infty$ norm for
outlier removal, where measurements with maximal residuals are thrown away. The authors
prove that at least one outlier is removed at each iteration, meaning that all outliers
will be rejected in a finite number of iterations. However the method is not efficient,
since one need to solve a sequence of SOCPs. Observing that many fixed-dimensional
$L_\infty$ geometry problems are actually instances of LP-type problem, an LP-type
framework was proposed for the multi-view triangulation problem with outliers \cite{Hongdong07}.

Recently, the Lagrange dual problem of the $L_\infty$ minimization
problem posed in \cite{Hartley04b} was derived in \cite{Olsson2010}.
%
%
%
To further
boost the efficiency of the method, the authors of \cite{Olsson2010}
proposed an $L_1$-minimization algorithm for outlier removal. While
the aforementioned methods add a single slack variable and repeatedly
solve a feasibility problem, the $L_1$ algorithm adds one slack
variable for each residual and then solves a single convex program.
While efficient, this method is only successful on data drawn from
particular statistical distributions.

Robust statistical techniques, including the aforementioned robust
regression and outlier removal methods, can significantly improve the
performance of their classic counterparts. However, they have rarely
been applied in image analysis field, to problems such as visual
recognition, due to their computational expense. The M-estimator
method is utilized in \cite{RLRC2012} for face recognition and
achieved high accuracy even when illumination change and pixel
corruption were present. In \cite{Fidler06}, the authors propose a theoretical
framework combining reconstructive and discriminative subspace methods
for robust classification and regression. This framework acts on
subsets of pixels in images to detect outliers.

The reminder of this paper is organized as follows. In
Section~\ref{SEC:linf} we briefly review the $L_2$ and $L_\infty$
problems. In Section~\ref{SEC:remove}, the main outlier removal
algorithm is presented.  In Section~\ref{SEC:fast_alg}, we formulate
the $L_\infty$ norm minimization problem into a set of small
sub-problems which can be solved with high efficiency. We then apply
the outlier removal technique in Section~\ref{SEC:Exp} to several
visual recognition applications. Finally the conclusion is given in
Section~\ref{SEC:conclusion}.

\section{The $L_2$ and $L_\infty$ norm minimization problems}
\label{SEC:linf}
In this section, we briefly present the $L_\infty$ norm minimization problem in the form we use in several recognition problems. Let us first examine the $L_2$ norm minimization problem,
\begin{equation}
    \label{EQ:LS}
    \min_{\Bbeta }  \;  \sum_{i=1}^n
        ( \Ba_i \Bbeta - y_i )^2,
    \end{equation}
for which we have a closed-form solution\footnote{The closed-form solution can only be obtained when $\BA$ is over-determined, i.e., $n \geq d$. When $n < d$, one can solve the multicollinearity problem by ridge regression, or another variable selection method, to obtain a unique solution.}
\begin{equation}
    \label{EQ:2}
    {\Bbeta^\ast}_{ls} = (  \BA ^\T \BA )^{-1} \BA ^\T \By,
\end{equation}
where  $ \BA = [ \Ba_1^\top; \dots; \Ba_n^\top ] \in {\mathbb R}^{ n \times d} $
        is the measurement data matrix, composed of rows $ \Ba_i \in {\mathbb R}^{d} $, and usually $n \gg d$.
        The model's response is represented by the vector $ \By \in {\mathbb R}^n $ and 
        $ \Bbeta \in {\mathbb R}^d $
        stores the parameters to be estimated. 
 Note that in our visual recognition applications, both $\By$ and the columns of $\BA$ are images flattened to vectors.  According to the linear subspace assumption \cite{Basri00lambertianreflectance}, a probe image can be approximately represented by a linear combination of the training samples of the same class: $\By \approx \BA\Bbeta$. Due to its simplicity and efficacy, the linear representation method is widely used in various image analysis applications, e.g., \cite{LRC10,ShiEHS11,CRCzhanglei2011,Wright09,Yang20121104,Iris2011}.
 
The $L_2$ norm minimization aims to minimize the sum of squared residuals \eqref{EQ:LS}, where the terms $f_i\, (\Bbeta) = ( \Ba_i \Bbeta - y_i )^2,\; i \in I = \{ 1, \ldots, n \}$,
        are the squared residuals.  
$L_2$ norm minimization is simple and efficient, however it utilizes the entire data set and therefore can be easily influenced by outliers.

Instead of minimizing the sum of squared residuals, the $L_\infty$-$L_2$ norm minimization method seeks to minimize only the maximal residual, leading to the following formulation:
\begin{equation}
\label{EQ:linf}
\min_{\Bbeta }  \;  \max_i \;
        ( \Ba_i \Bbeta - y_i )^2, i \in I.
\end{equation}

This equation has no closed-form solution, however it can be easily
reformulated into a constrained formulation, with an auxiliary variable:
\begin{align}
\label{EQ:SOCP}
\min_{\Bbeta} \; & s \notag \\
\st \; & ( \Ba_i \Bbeta - y_i )^2 \leq s, \, \forall i \in I;
\end{align}
which is clearly a SOCP problem. If we take the absolute value of the residual in \eqref{EQ:linf}, we obtain
\begin{equation}
\label{EQ:abs}
\min_{\Bbeta }  \;  \max_i \;
        | \Ba_i \Bbeta - y_i |, \quad i \in I.
\end{equation}
leading to an LP problem
\begin{align}
\label{EQ:LP}
\min_{\Bbeta} \; & s \notag \\
\st \; &  \; \Ba_i \Bbeta - y_i  \; \leq s, \notag \\
& \; \Ba_i \Bbeta - y_i  \; \geq -s,\quad \text{for all } i \in I.
\end{align}

A critical advantage of the $L_\infty$ norm cost function is that it
has a single global minimum in many multi-view geometry problems
\cite{Hartley04b,Hongdong07}. Unfortunately, like $L_2$ norm
minimization, the $L_\infty$ norm minimization method is also
vulnerable to outliers. Moreover, minimizing the $L_\infty$ norm fits
to the outliers, rather than the data truly generated by the
model \cite{Sim2006}. Therefore, it is necessary to first reject
outliers before the estimation.

\section{Outlier removal via maximum residual }
\label{SEC:remove}
In \cite{Sim2006}, outlier removal is conducted in an iterative fashion by first minimizing the $L_\infty$ norm, then removing the measurements with maximum residual and then repeating. The measurements with maximum residual are referred to as the \textit{support set} of the minimax problem, i.e., 
\begin{equation}
\label{EQ:supp}
I_{supp} = \{\, i \in I \;| \; f_i ( {\Bbeta^\ast} ) = \delta_{opt} \},
\end{equation}
where 
$\delta_{opt} = \min_\Bbeta \max_{i \in I} f_i (\Bbeta) $ 
is the optimum residual of the minimax problem.
 \comment{This seems nonsensical? $L_2$ doesn't have anything to do with the equation above--- added description in ().} The outlier removal strategy does not work well for the general $L_2$ minimization problems (i.e., $\delta_{opt} = \min_\Bbeta \sum_i f_i (\Bbeta)$), because the outliers are not guaranteed to be included in the support set. In contrast, this strategy is valid for the $L_\infty$ minimization problems. For problem \eqref{EQ:linf} or \eqref{EQ:abs}, it is proved by the following theorems that the measurements with largest residual must contain at least one outlier. 

Suppose the index vector $I$ is composed of $I_{in}$ and $I_{out}$, the inlier and outlier sets respectively, and there exists $\delta_{in}$ such that $\min_\Bbeta \max_{i \in I_{in}} f_i (\Bbeta) < \delta_{in} $.  Then we have the following theorem.

\begin{thm}
\label{thm:1} 
\cite{Sim2006}  Consider the $L_\infty$ norm minimization problem \eqref{EQ:linf} or \eqref{EQ:abs} with the optimal residual  $\delta_{max} = \min_\Bbeta \max_{i \in I} f_i (\Bbeta)$. If there exists an inlier subset $I_{in} \subset I$ for which $\min_\Bbeta \max_{i \in I_{in}} f_i (\Bbeta) < \delta_{in} < \delta_{opt}$, then  the support set $I_{supp}$ must contain at least one index $i \in I_{out}$, that is, an outlier.
\end{thm}
Following Theorem 2 in \cite{Sim2006}, Theorem \ref{thm:1} can be easily proved based on the following theorem. 

\begin{thm} 
\label{thm:2}
\textit{ If $i_0$ is not in the support set $I_{supp}$ for the minimax problem \eqref{EQ:linf} or \eqref{EQ:abs}, then removing the measurement with respect to $i_0$ will not decrease the optimal residual $\delta_{opt}$. Formally,
if $i_0 \notin I_{supp}$, then}
\begin{equation}
\min_\Bbeta \max_{i \in I \setminus {i_0}} f_i (\Bbeta ) = \min_\Bbeta \max_{i \in I}  f_i (\Bbeta ) = \delta_{opt}.
\end{equation}
\end{thm}
\begin{proof}
It is not difficult to verify that both the residual error function $(\Ba_i \Bbeta - y_i )^2$ and $|\Ba_i \Bbeta - y_i|$ are convex, and therefore also quasiconvex \cite{Boyd}.
Furthermore these two error functions are also strictly convex then also strictly quasiconvex.
 Then due to Corollary 1 in \cite{Sim2006} (omitted here), the theorem holds.
\end{proof}

\begin{algorithm}[t]
\caption{Outlier removal using the $L_\infty$ norm}
\label{alg:1}
\begin{algorithmic}[1]

\STATE \textbf{Input:}
 the measurement data matrix $\BA \in {\mathbb R}^{ n \times d} $; the response vector $\By \in {\mathbb R}^{ n}$;
 outlier percent $p$.
 \STATE \textbf{Initialization:} $l \leftarrow 0$; number of measurements to be removed $L \leftarrow \lfloor n \times p\rfloor$;  index $I_{t} \leftarrow  \{1, \ldots, n\}$.
 \WHILE{ $l < L$}
 \STATE Solve the $L_\infty$ norm minimization problem: $\delta_{opt} = \min_\Bbeta \max_{i \in I_{t}} f_i (\Bbeta)$;  get the support set $I_{supp}$ via equation \eqref{EQ:supp}.
 \STATE Remove the measurements with indices in  $I_{supp}$, i.e., $I_{t} \leftarrow I_t \setminus I_{supp}$.
\STATE \textbf{Remedy}. Solve the minimax problem again with the new index $I_t$ and get the optimal residual $\delta_{opt}$ and parameter ${\Bbeta^\ast}$. Move the indices $I_r$ = $\{ i \in I_{supp} | f_i ({\Bbeta^\ast}) < \delta_{opt}\}$ back to $I_t$. 
 \STATE $l \leftarrow l + |I_{supp}|$.
 \ENDWHILE
 \STATE \textbf{Output:} $\BA_t$ and $\By_t$ with measurement index $i \in I_{t}$.

\end{algorithmic}
\end{algorithm}

At each iteration, we first obtain the optimal parameters
$\Bbeta^\ast$ by solving \eqref{EQ:linf} or \eqref{EQ:abs} and then
remove the measurements (pixels in images) corresponding to largest
residual.  If we continue the iteration, all outliers are eventually 
removed. 

As with all outliers removal processes, there is a risk that
discarding a set of outliers will remove some inliers at the same
time.  In this framework, the outliers are individual pixels, which
are in good supply in visual recognition applications. For example, a
face image will typically contain hundreds or thousands of pixels.
Removing a small fraction of the good pixels is therefore unlikely to
affect recognition performance. However, if too many pixels are
removed, the remaining pool may be too small for successful
recognition. Therefore, we propose a process to restore incorrectly
removed pixels where possible, as part of the overall outlier removal
algorithm list in Algorithm \ref{alg:1}. 
In practice, the heuristic remedy step does improve the performance of
our method on
the visual recognition problems in our experiments. Also note that it
is impossible that all points in the support set are  moved back in
step 6, which is because, based on Theorem \ref{thm:2} we can prove
that
\begin{align*}
\max_{i \in I_{supp}}f_i(\beta^{\ast})
& \geq \min_{\beta}\max_{i \in I_{supp}}f_i(\beta) 
\\
&=
\min_{\beta}\max_{i \in I_t \cup I_{supp}}f_i(\beta) > \min_{\beta}\max_{i \in I_t}f_i(\beta).
\end{align*}

\section{ A fast algorithm for the $L_\infty$ norm minimization problem }
\label{SEC:fast_alg}
Recalling Theorem~\ref{thm:2}, we may remove any data not in the
support set without changing the value of the $L_\infty$ norm. This
property allows us to subdivide the large problem into a set of
smaller sub-problems. We will proceed by first presenting a useful
definition of pseudoconvexity:

\textbf{Definition 1.} \textit{A function $f(\cdot)$ is called pseudoconvex
if $f(\cdot)$ is differentiable and $\nabla f ( \bar{x} )(x - \bar{x}) \geq 0$ implies $f ( x ) \geq f ( \bar{x} )$.}

In this definition, $f(\cdot)$ has to be differentiable. However the notion of pseudoconvexity can be generalized to non-differentiable functions \cite{pseudoconvex2001}:

\textbf{Definition 2.} \textit{A function $f(\cdot)$ is called pseudoconvex if for all $x, y \in \text{dom}\,(f)$:}
\begin{align}
\exists x^\ast \in \partial f(x): \quad \langle x^\ast, y - x\rangle \;\geq 0 \notag\\
\Rightarrow \forall z \in [x, y): f(z) \leq f(y),
\end{align}
\textit{where $\partial f(x)$ is subdifferential of $f(x)$.}

Both of these definitions share the property that any local minimum of a pseudoconvex function is also a global minimum. Based on the first definition, it has been proved that if the residual error functions $f_i(\cdot)$ are pseudoconvex and differentiable,
the cardinality of the support set is not larger than $d + 1$ \cite{Olsson08,Hongdong09}. 
Following the proof in \citep{Hongdong09}, one can easily validate the following corollary:

\textbf{Corollary 1.}
\textit{For the minimax problem with pseudoconvex residual functions (differentiable or not), there must exist a subset $I_s \subset I, |I_s| \leq d + 1$ such that }
\begin{equation}
f_{I_s} (\beta^\ast) =  \min_\Bbeta \max_{i \in I_s} f_i(\Bbeta ) = f_{I} (\beta^\ast) = \min_\Bbeta \max_{i \in I}  f_i (\Bbeta ).
\end{equation}

It is clear that the squared residual functions $f_i (\Bbeta ) = ( \Ba_i \Bbeta - y_i )^2$ in \eqref{EQ:linf} are convex and differentiable,  hence also,  pseudoconvex. 
The absolute residual function \eqref{EQ:abs}: $f_i (\Bbeta ) = | \Ba_i \Bbeta - y_i |$ is sub-differentiable. It is easy to verify that the only non-differentiable point, the origin, satisfies the second definition. Therefore function \eqref{EQ:abs} also satisfies Corollary 1.

The above corollary says that we can solve a sub-problem with at most $d + 1$ measurements without changing the estimated solution to the original minimax problem. However before solving the sub-problems, we should first determine the support set. We choose to solve a set of small sub-problems using an optimization method called \textit{column generation} (CG) \cite{Lubbecke05}. The CG method adds one constraint at a time to the current sub-problem until an optimal solution is obtained.

The process is as follows: We first choose $d + 1$ measurements not contained in the support set. These data are then used to compute a solution and residuals for all the data in the main problem are determined. Then the most violated constraint, or the measurement corresponding to the largest residual, is added to the sub-problem. The sub-problem is then too large, therefore we solve the sub-problem again (now with size $d + 2$) and remove an inactive measurement. Through this strategy, the problem is prevented from growing too large, and violating Corollary 1. When there are no violated constraints, we have obtained the optimal solution.

The proposed fast method is presented in Algorithm \ref{alg:2}.  We
divide the data into an active set, corresponding to a sub-problem,
and the remaining set with the $L_2$ norm minimization. This algorithm
allows us to solve the original problem with the measurement matrix of
size $n \times d$, by solving a series of small problems with size $(d
+ 1) \times d$ or $(d + 2) \times d$. In most visual recognition
problems, $n \gg d$.  Typically the algorithm  converges in less
than 30 iterations in all of our experiments. 
We will show that this strategy radically improves
computational efficiency.

For maximal efficiency, we choose to solve the LP problem, \eqref{EQ:LP}, and utilize the code generator \texttt{CVXGEN} \cite{CVXGEN} to generate custom, high speed solvers for the sub-problems in algorithm \ref{alg:2}.
\texttt{CVXGEN} is a software tool that automatically generates
customized C code for LP or QP problems of modest size.
\texttt{CVXGEN} is less effective on large problems (e.g., $>2000$
variables). However, in Algorithm \ref{alg:2} we convert the original problem into a set of small sub-problems, which can be efficiently solved with the generator. 
\texttt{CVXGEN} embeds the problem size into the generated code, restricting it to fixed-size problems \cite{CVXGEN}. The proposed method is only ever solves problems of size $d +1$ or $d + 2$, enabling the use of \texttt{CVXGEN}.

\begin{algorithm}[t]
\caption{A fast algorithm for the $L_\infty$ problem}
\label{alg:2}
\begin{algorithmic}[1]

\STATE \textbf{Input:}
The measurement data matrix $\BA \in {\mathbb R}^{ n \times d} $; the response vector $\By \in {\mathbb R}^{ n}$; maximum iteration number $l_{max}$.
 
\STATE \textbf{Initialization:}
 Initialize the active set $I_s$ with indices corresponding to the largest $d + 1$ absolute residuals from vector \big\{ $f_i(\Bbeta_{ls})$ \big\},  $i \in I$ using the LS solution  as in \eqref{EQ:2}; set $I_r \leftarrow I \setminus I_s$; set iteration counter $l \leftarrow 0$.

\STATE Solve the $L_\infty$-minimization sub-problem 
\begin{equation}
\label{EQ:sub}
\Bbeta^{\ast} = \arg\min_\Bbeta \max_{i \in I_s} f_i (\Bbeta )
\end{equation} 
and set $t^{\ast} \leftarrow \max_{i \in I_s} f_i (\Bbeta^{\ast} )$.
\WHILE{$l < l_{max}$}

\STATE Get the most violated measurement from the remaining set $I_r$:  $i_m = \arg\max_{i \in I_r} f_i (\Bbeta^{\ast} )$  
\STATE Check for optimal solution: \\
\textbf{if} {$f_{i_m} (\Bbeta ) \leq t^{\ast}$},\;  \textbf{then} break (problem solved).

\STATE Update the active and remaining set:\\
 $I_s \leftarrow I_s \cup \{i_m\}$, \; $I_r \leftarrow I_r \setminus \{i_m\}$.

\STATE Solve the $L_\infty$-minimization sub-problem in \eqref{EQ:sub}.
\STATE Move the inactive measurement index $i_d = \arg\min f_i (\Bbeta^{\ast} )$
to the remaining set:  $I_s \leftarrow I_s \setminus \{i_d\}$, \; $I_r \leftarrow I_r \cup \{i_d\}$

\STATE $l \leftarrow l + 1$.
\ENDWHILE
\STATE \textbf{Output:}
$\Bbeta^{\ast}$.

\end{algorithmic}
\end{algorithm}

\section{Experimental Results}
\label{SEC:Exp}
In this section, we  first illustrate the effectiveness and efficiency
of our algorithm on several classic geometric model fitting problems.
Then the proposed method is evaluated on face recognition problems
with both artificial and natural contiguous occlusions. Finally, we
test our method on the iris recognition problem, where both
segmentation error and occlusions are present. For comparison, we also
evaluate several other representative robust regression methods on
face recognition problems.

Once the outliers have been removed from the data set, any solver can be used to obtain the final model estimate. 
We implemented the original minimax algorithm using Matlab package
\texttt{CVX} \cite{cvx} with SeDuMi solver \cite{sedumi99} while the
proposed fast algorithm was implemented using solvers generated by
\texttt{CVXGEN} \cite{CVXGEN}. All experiments are conducted in Matlab
running on a PC with a Quad-Core 3.07GHz CPU and 12GB of RAM, using
\texttt{mex} to call the solvers from \texttt{CVXGEN}. Note that the
algorithm makes no special effort to use multiple cores, though Matlab
itself may do so if possible.

\subsection{Geometric model fitting}
\subsubsection{Line fitting}
\label{subsub:line}

\begin{figure*}[t!]
\centering
{\includegraphics[width = 0.28\textwidth]{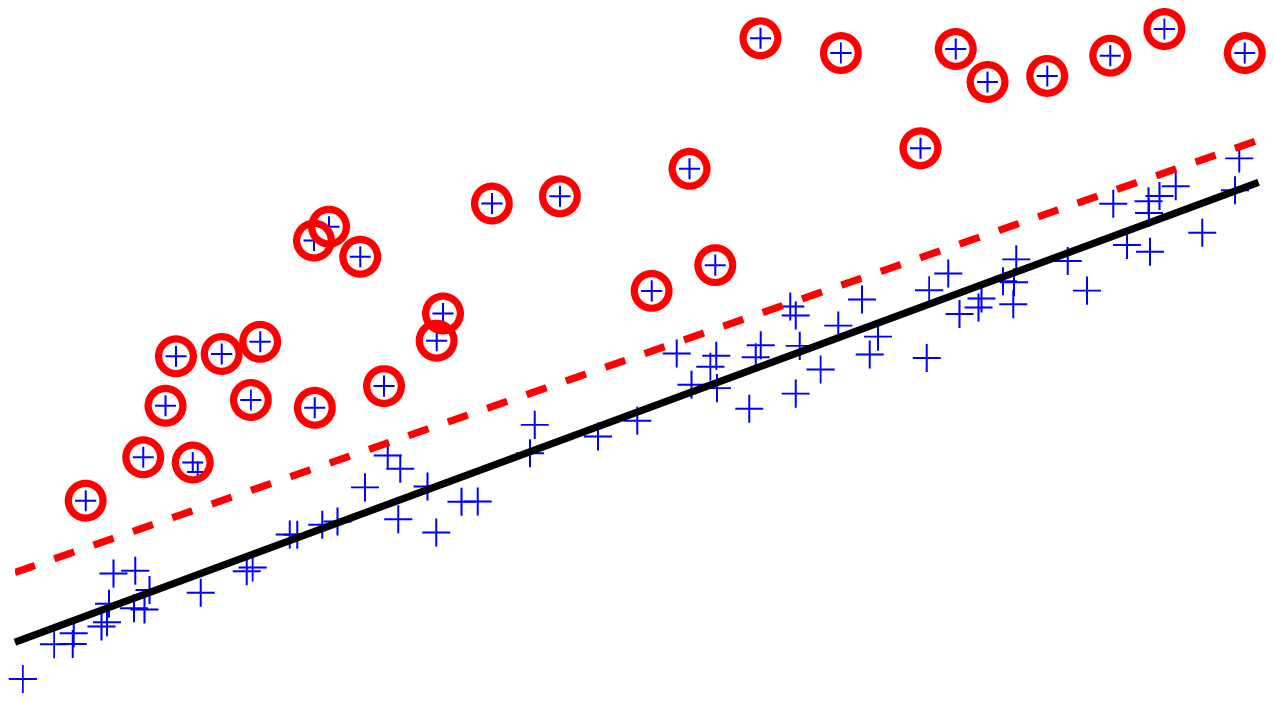} } \quad\quad\quad\quad
{\includegraphics[width = 0.28\textwidth]{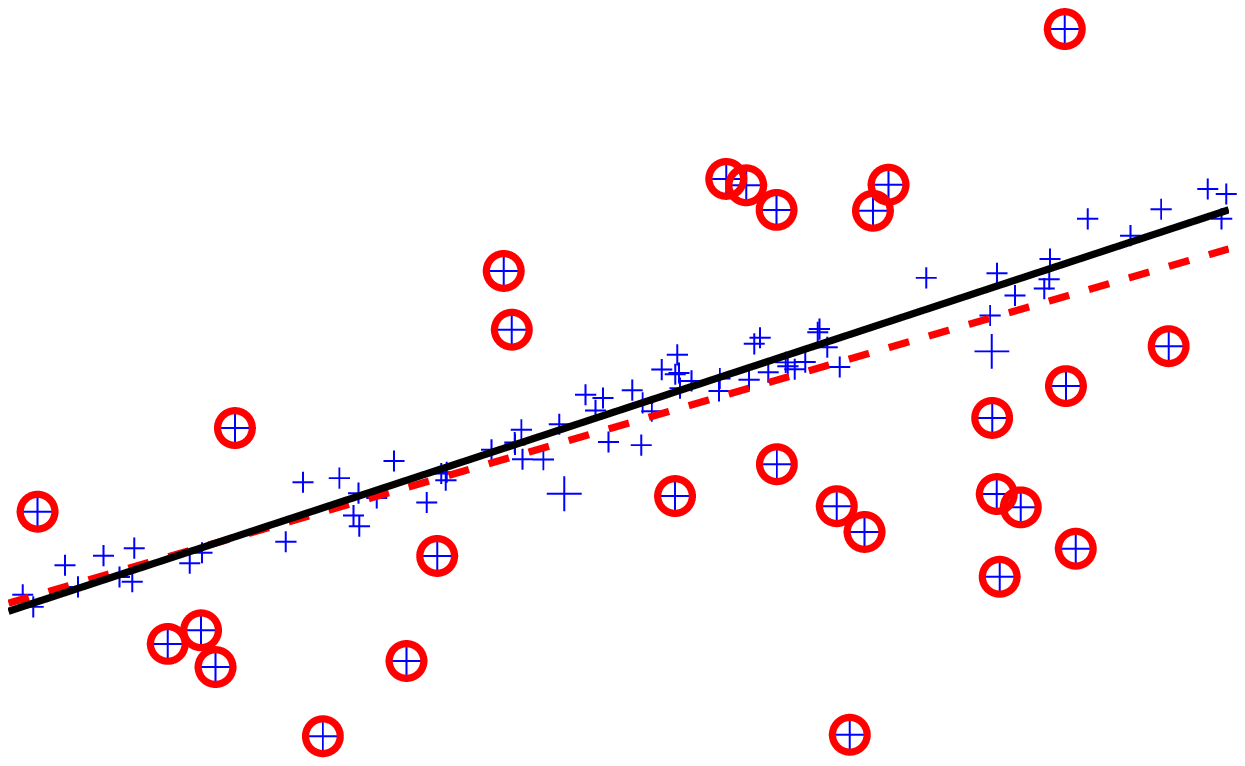}      }
\caption{Two examples demonstrating the performance of our algorithm
on data contaminated on one-side (left) and two-sides (right).
Outliers detected by our algorithm are marked with circles. The
solid black line is the result after outlier removal and the red
dashed line shows the result with outliers included.  In these two
cases, we set $n = 100$ and $k = 70$, where $n$ and $k$ are described
in Section~\ref{subsub:line}.
}
\label{Fig:SidedContamination}
\end{figure*}
Figure~\ref{Fig:SidedContamination} shows estimation performance when our algorithm is used for outlier removal and the line is subsequently estimated via least squares, on data generated under two different error models.

%

We generate $ \BA \in \mathbb{R}
^ {n \times d} $ randomly and  $\Bbeta \in \mathbb{R} ^ d$ randomly. We then set the first $k$ error
terms $ \epsilon _ j, j \in [1, \ldots, k]$ as independent standard normal random variables. We set
the last $n - k$ error terms $ \epsilon _ j, j \in [k + 1, \ldots, n]$ as independent chi squared random variables with $5$ degrees of
freedom. We also test using the two-sided contamination model which sets the sign of the last $n -
k$ variables randomly such that the outliers lie on both side of the true regression line. In both
cases we set $\By = \BA\Bbeta + \Bepsilon$.

As can be seen in Figure~\ref{Fig:SidedContamination}, our method detects all of the outliers and
consequently generates a line estimate which fits the inlier set well for both noise models, whilst
the estimate obtained with the outliers included achieves a reasonable estimate only for the two-sided contamination case, where the outliers are evenly distributed on both sides of the line. 

\subsubsection{Ellipse fitting}
\label{subsub:ellipse}
An example of the performance of our method applied to ellipse fitting is shown in Figure~\ref{fig:circle}.
$100$ points were sampled uniformly around the perimeter of an ellipse centred at $(0,0)$, and where then perturbed via offset drawn from $\mathcal{N}(0,1)$.
%
$30$ outliers were randomly drawn from an approximately uniform distribution within the bounding box shown in Figure~\ref{fig:circle}.
The result of the method, again shown in Figure~\ref{fig:circle}, shows that our method has correctly identified the inlier and outlier sets, and demonstrates that the centre and radius estimated by our method are accurate.

\begin{table*}[th]
\centering
\begin{tabular}{ccccccccc}
\hline
\multirow{2}{*}{\textbf{Method}} & \multicolumn{8}{c}{\textbf{Number of observations}} \\
 &20 & 50 & 100 & 200 & 500 & 1000 & 2000& 10000\\
\hline
original & 0.185 & 0.454 & 0.89 & 1.926& 5.093& 11.597 & 29.968& 313.328 \\
fast & 0.002 &0.005 & 0.011&0.031 & 0.083 &0.164&0.395 &4.127\\
\hline
\end{tabular}
\caption{Computation time (in seconds) comparison of the original and fast algorithms implemented with different solvers. For line fitting problem, we fixed $d = 2$ and the observation number is increased from 20 to 10000.}
\label{Tab:time_line}
\end{table*}

\begin{figure}[b!]
\centering
\includegraphics[width = 0.4\textwidth]{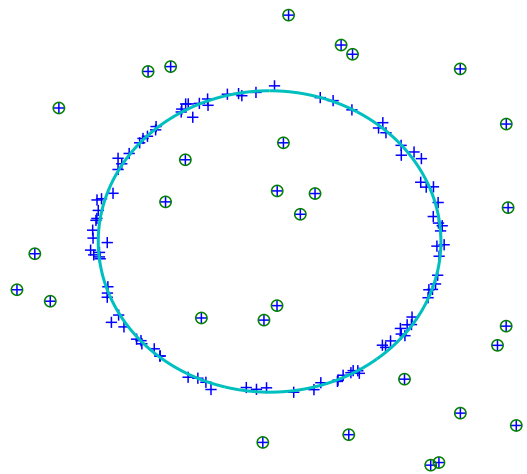} 
\caption{An example of the performance of our method in ellipse fitting. Points identified as
outliers are marked by circles.}
\label{fig:circle}
\end{figure}

\subsubsection{Efficiency}
Next we compare the computational efficiency of the standard
$L_\infty$ norm outlier removal process and of the proposed fast
algorithm. 

For the line fitting problem we generate the data using the scheme described previously. We initially fix the data dimension $d = 2$ and increase the problem size $n$ from 20 to 10000. The outlier fraction, $k$, is set 90\% of $n$.
A comparison of the running time for these two algorithms are shown in Table~\ref{Tab:time_line}. The fast algorithm finishes 70 to 80 times faster than the original algorithm. Specifically, with dimension 10000 the fast algorithm finishes in approximately 4 seconds, whilst the original algorithm requires more than 5 minutes.
In this case, the proposed fast approach is about {\em 80 times
faster} than the conventional approach. 

Second, we fix the number of observations $n = 200$ and vary the data dimension $d$ from 2 to 10. Execution times are shown in Table~\ref{Tab:time_line_d}. Consistent with the last experiment, the fast algorithm completes far more rapidly than the original algorithm in all situations. 
When $d = 2$, the proposed fast algorithm is faster than the original algorithm by more than 60 times. With a larger dimension $d = 10$, the proposed fast algorithm takes only 0.128 seconds to complete while the original algorithm requires more than 2 seconds. 

\begin{table}
\centering
\begin{tabular}{cccccc}
\hline
\multirow{2}{*}{\textbf{Method}} & \multicolumn{5}{c}{\textbf{Feature dimension}} \\
& 2 & 4 & 6 & 8 & 10 \\
\hline
original & 1.926&1.984 &1.994 & 2.016&2.039 \\
fast &　0.031 &0.045 & 0.068&0.096 & 0.128\\
\hline
\end{tabular}
\caption{Computation time (in seconds) comparison of the original and fast algorithms implemented with different solvers. Data is generated as 200 observations with dimension varying from 2 to 10.}
\label{Tab:time_line_d}
\end{table}


\subsection{Robust face recognition}
\label{Sec:face_rec}
In this section, we test our method on face recognition problems from 3 datasets: AR \cite{AMM98}, Extended Yale B \cite{GeBeKr01}, and CMU-PIE \cite{Sim03thecmu}. A range of state-of-the-art algorithms are compared to the proposed method. Recently, sparse representation based classification (SRC) \cite{Wright09} obtained an excellent performance for robust face recognition problems, especially with contiguous occlusions. The SRC problem solves $\min  \Vert \Bbeta \Vert_1, \st \|\By - \BA\Bbeta\|_2 \leq \epsilon$, where $\BA$ is the training data from all classes, $\Bbeta$ is the corresponding coefficient vector and $\epsilon > 0$ is the error tolerance. To handle occlusions, SRC is extended to 
$
\min   \Vert \mathbf{\omega} \Vert_1, \st \|\By - \mathbf{B}\mathbf{\omega}\|_2 \leq \epsilon, 
$
where $\mathbf{B} = [\BA, \mathbf{I}]$, and $\mathbf{\omega} = [\Bbeta, \mathbf{e}]$. $\mathbf{I}$ and $\mathbf{e}$ are the identity matrix and error vector respectively. SRC assigns the test image to the class with smallest residual: $identity(\By) = \min_i \|\By - A\delta_i(\Bbeta)\|_2$. Here $\delta_i(\Bbeta)$ is a vector whose only nonzero entries are the entries in $\Bbeta$ corresponding to the $i$-th class.
We also evaluate the following two methods which are related to our method.
Most recently, a method called Collaborative Representation-based Classification (CRC) was proposed in \cite{CRCzhanglei2011} which relax the $L_1$ norm  to $L_2$ norm.  \comment{Is this right? Why do we care if it was only competitive? Is it better for some other reason? --- All these methods are state-of-the-art. I said they are competitive because sometimes one is better on a dataset, 
and another one may be better on another dataset. In the LRC and CRC paper, comparable or even better results are reported,
such as on the AR dataset.
All these three methods are recently proposed and related to our method. Our method first removes outliers and then apply LRC.}
Linear regression classification (LRC) \cite{LRC10} cast face
recognition as a simple linear regression problem: $\min  \Vert
\Bbeta_i \Vert_1, \st \By = \BA_i\Bbeta_i$, where $\BA_i$ and
$\Bbeta_i$ are the training data and representative coefficients with
respect to class $i$. LRC selects the class with smallest residual:
$identity(\By) = \min_i \|\By - A_i\Bbeta_i\|_2$.  Both CRC and LRC
achieved competitive or even better results than SRC
\cite{CRCzhanglei2011,LRC10} in some cases.

For the purposes of the face recognition experiments, outlier pixels are first removed using our method,
leaving the remaining inlier set to be processed by any regression based classifier.  In the
experiments listed below, LRC has been used for this purpose due to its computational efficiency.

Lots of robust regression estimators has been developed in the
statistic literature. In this section, we also compare other two
popular estimators, namely, Least median of squares (LMS) \cite{LMS1984}
and MM-estimator \cite{yohai1987high},
both of which have high-breakdown points and do not need to specify the number of outliers to be removed.
Comparison is conducted  on face recognition  problems with both artificial and natural occlusions. These methods are first used to estimate the coefficient $\beta$ and face images are recognized by the minimal residuals.


\begin{figure}
\centering
\includegraphics[width = 0.06\textwidth]{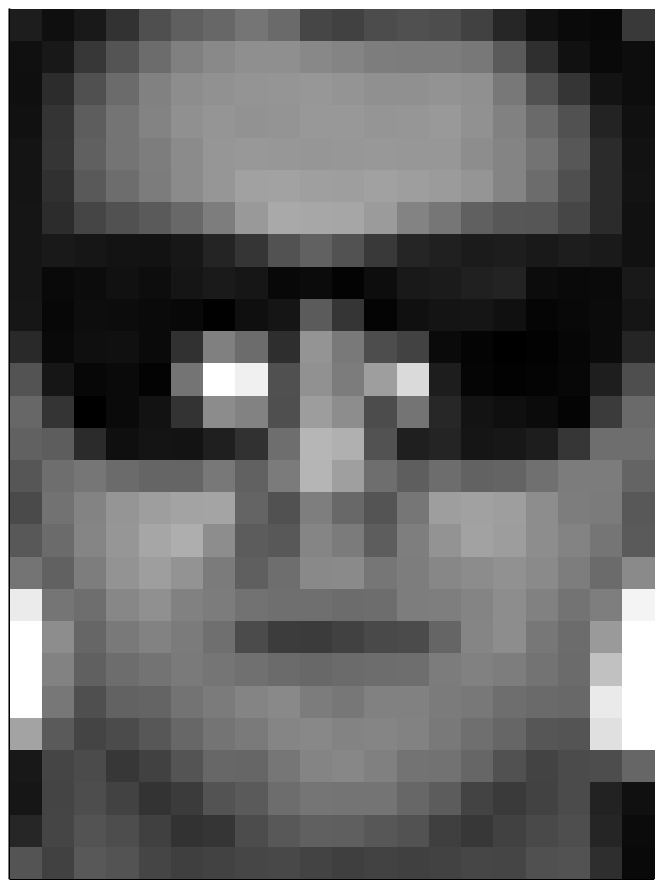}
\includegraphics[width = 0.06\textwidth]{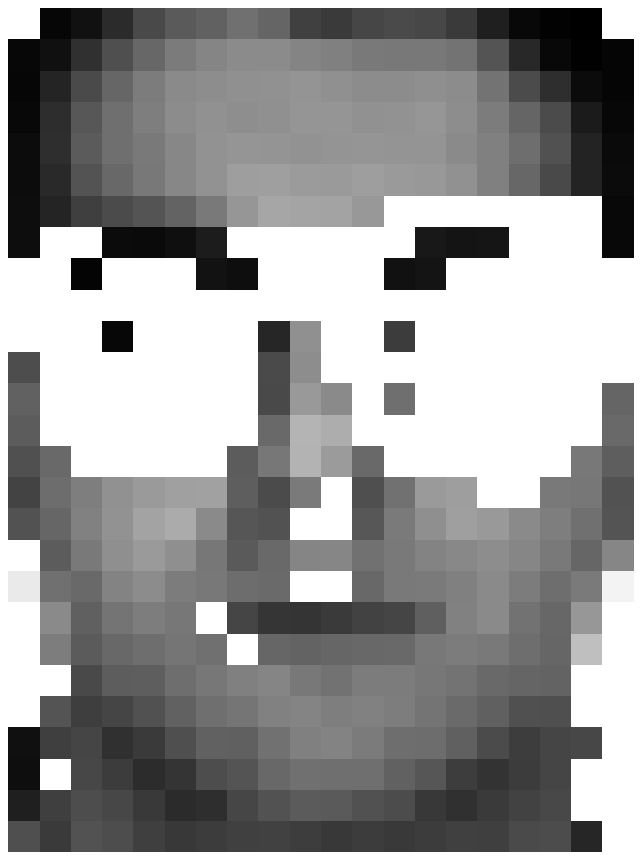}
\includegraphics[width = 0.06\textwidth]{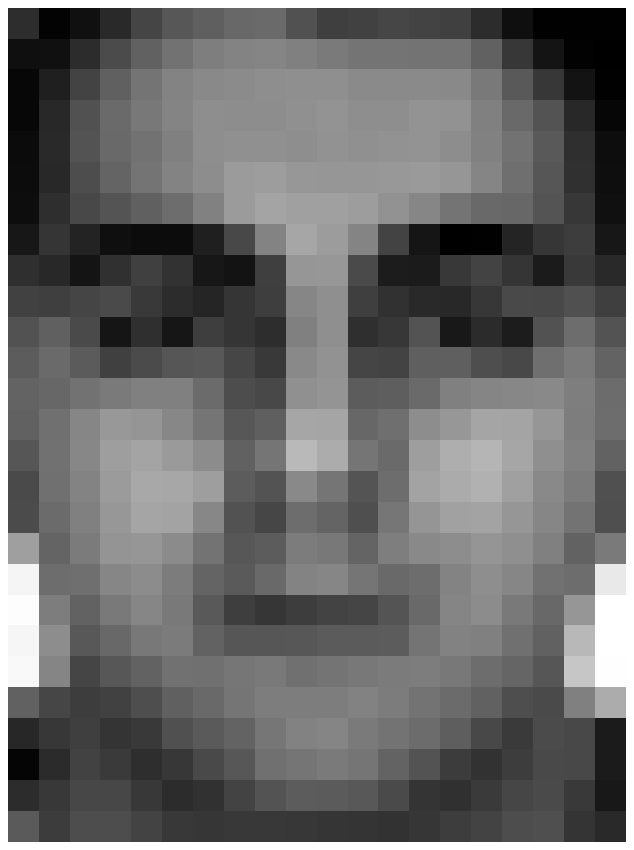}

\includegraphics[width = 0.06\textwidth]{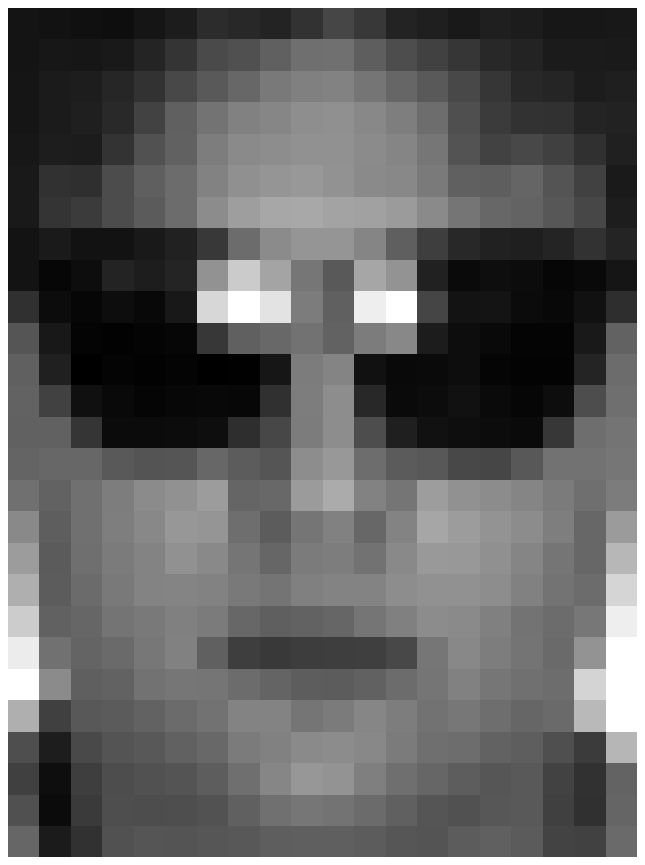}
\includegraphics[width = 0.06\textwidth]{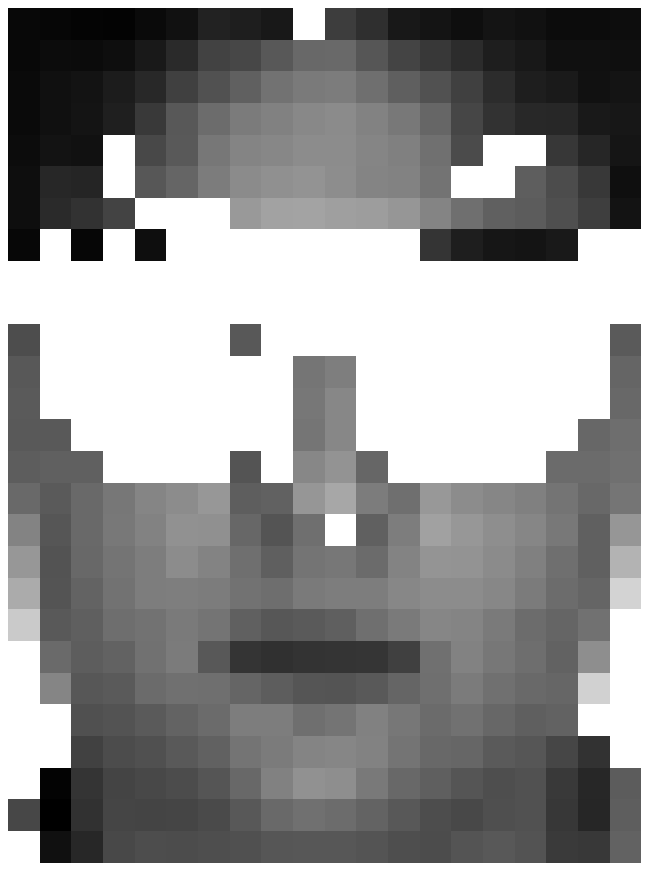}
\includegraphics[width = 0.06\textwidth]{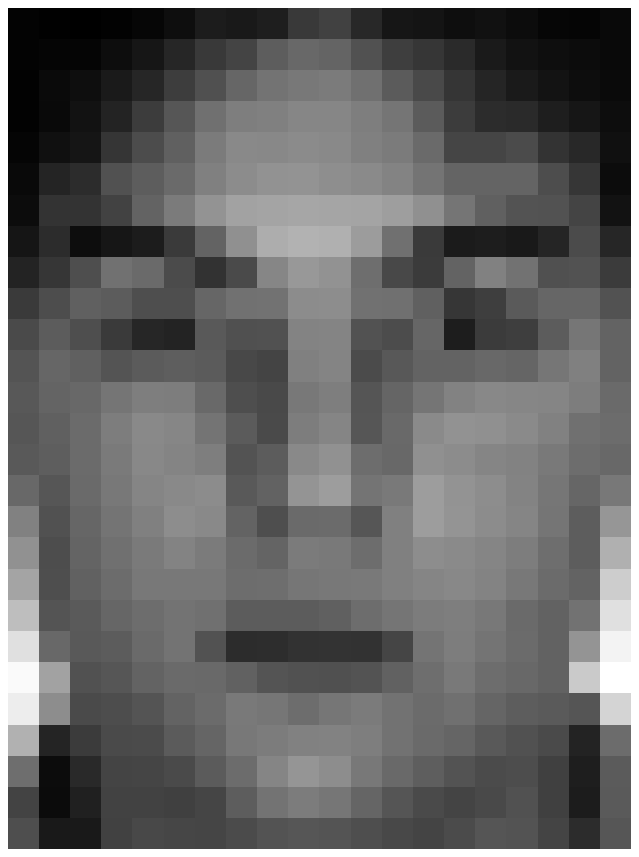}
\caption{Detecting the sunglasses occlusion of two example images with dimension $27 \times 20$ from the AR dataset. Each row shows an original image of a particular subject, followed by an image where outlying pixels have been automatically marked in white and finally a reconstructed image of the subject.}
\label{Fig:sunglass}
\end{figure}

\begin{table}[t]
\centering
\begin{tabular}{ccccc}
\hline
\multirow{2}{*}{\textbf{Method}} & \multicolumn{4}{c}{\textbf{Feature
dimension}} \\
   & 54 & 130 & 300 & 540 \\
\hline
LRC & 21.0\% & 38.5\% & 54.5\% & 60.0\% \\
SRC & \textbf{48.0}\% & 67.5\% & 69.5\% & 64.5\%\\
CRC & 22.0\% & 35.5\% & 44.5\% & 56.0\%\\
\hline
MM-estimator & 0.5\% & 8.5\% & 21\% & 24\% \\
LMS &  9\% & 25\% & 37.5\% & 48\%\\
our method & 43.0\% & \textbf{85.0\%} & \textbf{99.5\%} & \textbf{100\%} \\
\hline
\end{tabular}
\caption{Accuracy rates (\%) of different methods on the AR dataset
with sunglasses occlusion. The various feature dimension correspond to
downsampling the original $165 \times 120$ pixel images to $9 \times
6$, $13 \times 10$, $20 \times 15$ and $27 \times 20$, respectively.}
\label{Tab:sunglasses}
\end{table}

\subsubsection{Faces recognition despite disguise}
\label{subsub:disguise}
The AR dataset \cite{AMM98} consists of over $4000$ facial images from $126$ subjects ($70$ men and
$56$ women). For each subject $26$ facial images were taken in two separate sessions, $13$ per session. The images
exhibit a number of variations including various facial expressions (neutral, smile, anger, and scream), illuminations (left light on, right light on and all side lights on), and occlusion by sunglasses and scarves.
Of the $126$ subjects available, $100$ have been randomly selected for testing (50 males and 50 females) and the images cropped to $165 \times 120$ pixels. $8$ images of each subject with various facial
expressions, but no occlusions, were selected for training.
Testing was carried out on $2$ images of each of the selected subjects wearing sunglasses. Figure~\ref{Fig:sunglass} shows two typical images from the AR dataset with the outliers (30\% of all the pixels in the face images) detected by our method set to white. The reconstructed images are shown as the third and sixth images. 

The images were downsampled to produce features of $30$, $54$, $130$,
and $540$ dimensions respectively. Table~\ref{Tab:sunglasses} shows a
comparison of the recognition rates of various methods.  Our method
exhibits superior performance to LRC, CRC and SRC in all except the
lowest feature dimension case.   Specifically with feature dimension
540, the proposed method achieves a perfect accuracy 100\%, which
outperforms LRC, CRC and SRC by 40\%, 44\% and 35.5\% respectively.
MM-estimator and LMS failed to achieve good results on this dataset,
which is mainly because the residuals of outliers severely affect the
final classification although relatively accurate coefficients could be estimated.
In this face recognition application, the final classification is
based on the fitting residual. 
These results highlight the ability of our method for outlier removal, which can significantly improve the face recognition performance.

%

\begin{figure}
\centering
\includegraphics[height=1.2cm]{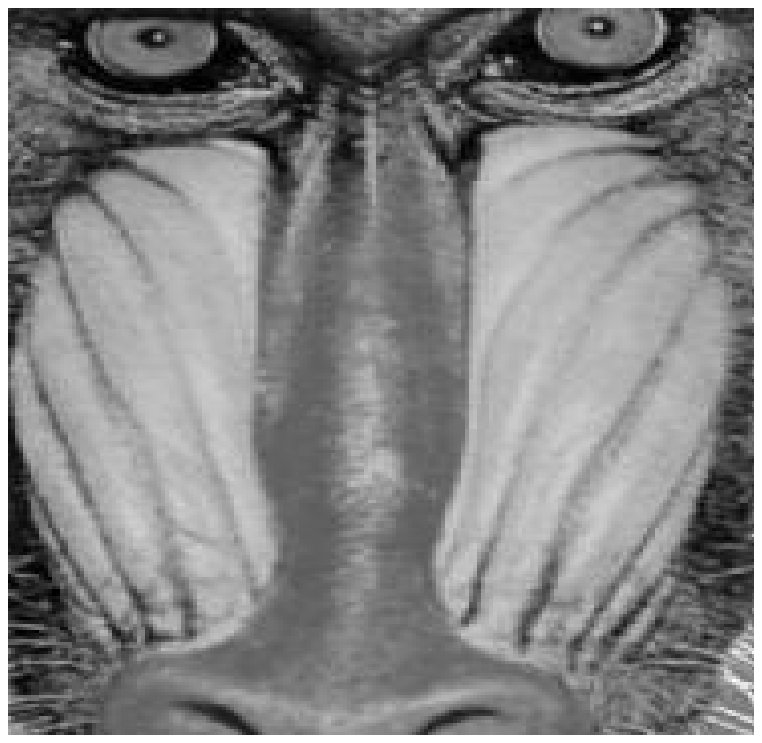} \quad\quad
\includegraphics[height=1.2cm]{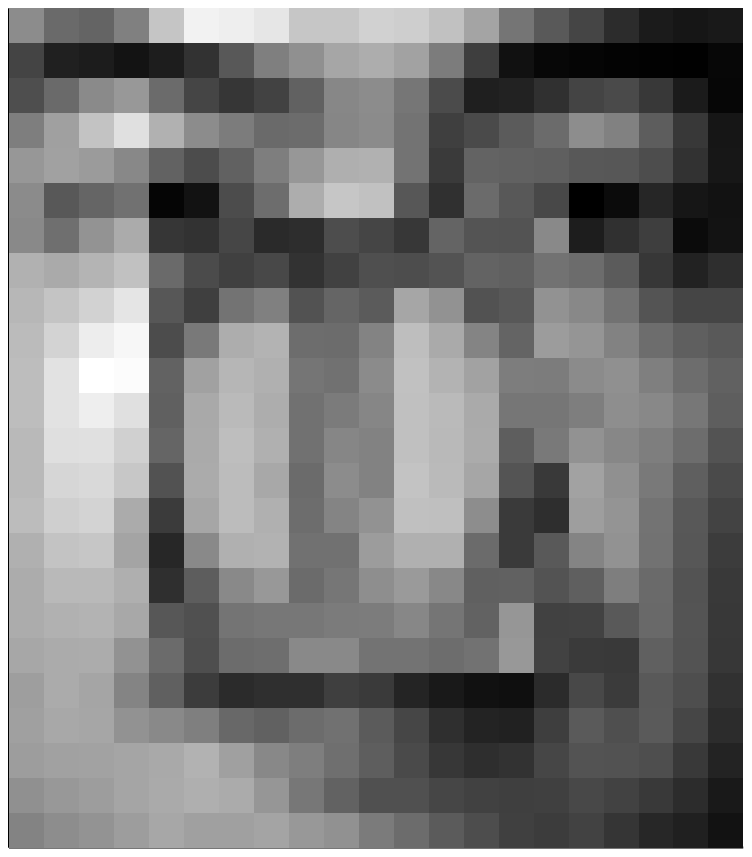}
\includegraphics[height=1.2cm]{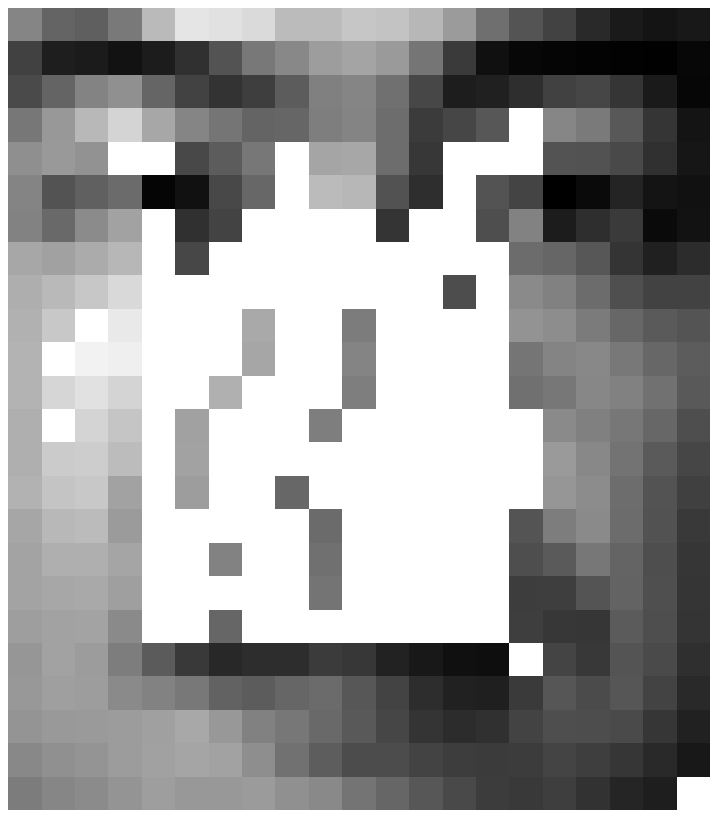}
\includegraphics[height=1.2cm]{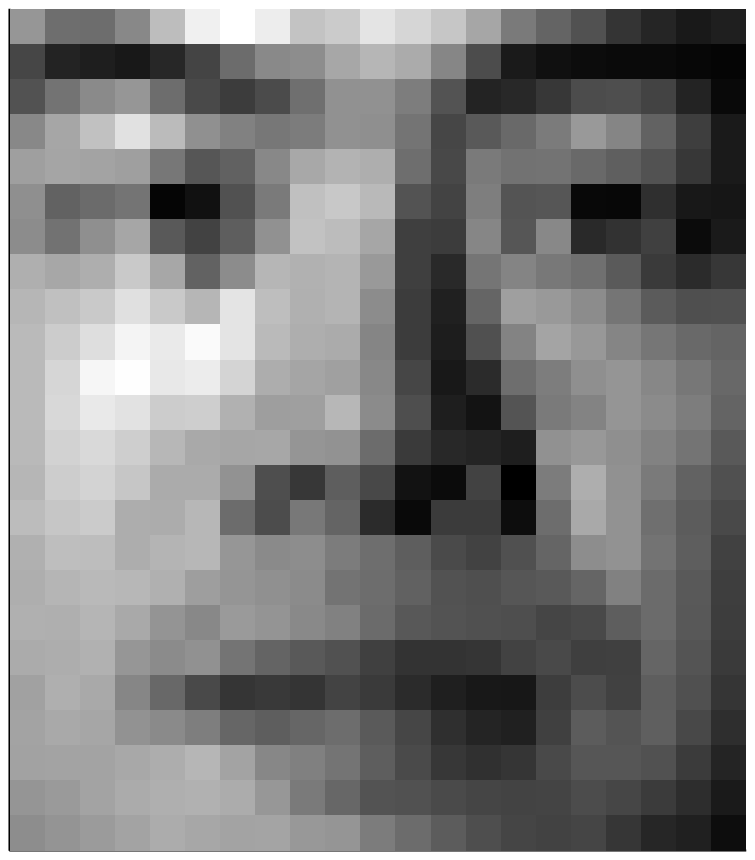}
\caption{Detecting the random placed square monkey face (the first image) in an example face image (the second image) from the Extended Yale B dataset. Outliers (30\% of the whole image) detected by our method are marked as white pixels (the third image) and  the reconstructed image is shown on the rightmost.}
\label{Fig:YaleB}
\end{figure}


\begin{table*}
\centering
\begin{tabular}{ccccccc}
\hline
\multirow{2}{*}{\textbf{Method}} & \multicolumn{6}{c}{\textbf{Occlusion rate} } \\
& 10\% & 20\% & 30\% & 35\% & 40\% & 50\%\\
\hline
LRC &  $98.5 \pm 0.18$ & $ 88.6 \pm 0.18$ & $ 68.8 \pm 3.3 $&  $60.4  \pm  3.2 $&$ 51.8  \pm 1.3$ & $41.7 \pm 1.7$\\

SRC & $ 99.0 \pm 0.3$  & $  94.8 \pm 1.2 $&$   77.2 \pm 1.5 $&  $ 67.3  \pm 1.1 $&  $56.2 \pm  2.2 $&$   45.9 \pm   1.5$\\

CRC & $ 98.6 \pm 0.7$ & $  90.3 \pm 1.7 $&  $ 74.9 \pm 2.6$ &   $67.2 \pm 1.0 $&$   58.3 \pm  1.8$& $ 48.3 \pm 2.0$ \\

our method & $\textbf{99.8}\pm \textbf{0.1} $& $\textbf{99.3}  \pm  \textbf{0.1} $&$  \textbf{98.2}  \pm \textbf{1.0}$ 
&$  \textbf{96}  \pm \textbf{1.2}  $& $ \textbf{91.4} \pm \textbf{0.7} $& $ \textbf{85.5}\pm \textbf{0.7}$ \\
\hline
\end{tabular}
\caption{{Mean and standard deviations of recognition accuracies (\%) in the presence of randomly placed block occlusions of images from the Extended Yale B dataset based on 5 runs results.}}
\label{Tab:YaleB_acc}
\end{table*}

\subsubsection{Contiguous block occlusions}
\label{subsub:contiguous}
In order to evaluate the performance of the algorithm in the presence of artificial noise and larger
occlusions, we now describe testing where large regions of the original image are replaced by pixels from another
source.
The Extended Yale B dataset \cite{GeBeKr01} was used as the source of the original images and consists of $2414$ frontal face images from $38$ subjects under various lighting conditions. The images are cropped and normalized to $192 \times 168$ pixels \cite{KCLee05}.
%
%
Following \cite{Wright09}, we choose subsets 1 and 2 (715 images ) for training and Subset 3 (451 images) for testing. In our experiment, all the images are downsampled to $24 \times 21$ pixels. We replace a randomly selected region, covering between 10\% and 50\% of each image, with a square monkey face. Figure~\ref{Fig:YaleB} shows the monkey face, an example of an occluded image, the outlying pixels detected by our method and a reconstructed copy of the input image.

Table~\ref{Tab:YaleB_acc} compares the average recognition rates of the different methods, averaged over five separate runs. Our proposed method outperforms all other methods in all conditions. With small occlusions, all methods achieve high accuracy, however, the performance of LRC, SRC and CRC deteriorate dramatically as the size of the occlusion increases. In contrast, our method is robust in the presence of outliers. In particular, with 30\% occlusion our method obtains 98.2\% accuracy while recognition rates of all the other methods are below 80\%. With 50\% occlusion, all other methods show low performances, while the accuracy for our method is still above 85\%. \comment{What does this mean? The curve is flatter? If its referring to the variance, the graph needs to be fixed because you can barely see them. --- I mean the Recognition Accuracy (not variance) is still very high even if with heavy occlusions (50\% by area). The figure is replaced with a table.} According to Table~\ref{Tab:YaleB_acc}, we can also see that the proposed method is more stable in the sense of accuracy variations, which is mainly because  the outliers are effectively detected.



\subsubsection{Partial face features on the CMU PIE dataset}
\label{SEC:PIE}
\begin{figure}
\centering
\includegraphics[height=1.2cm]{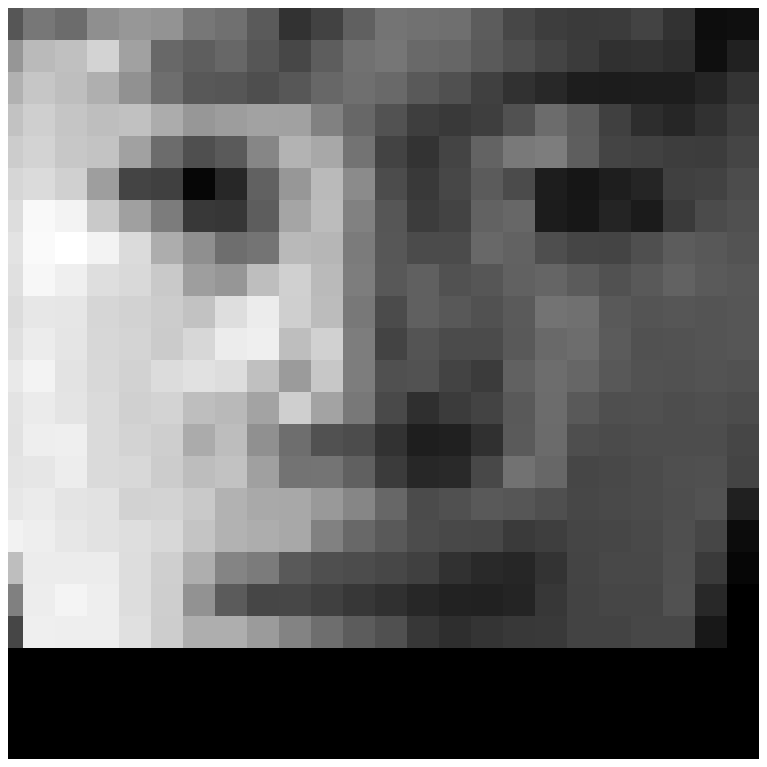}
\includegraphics[height=1.2cm]{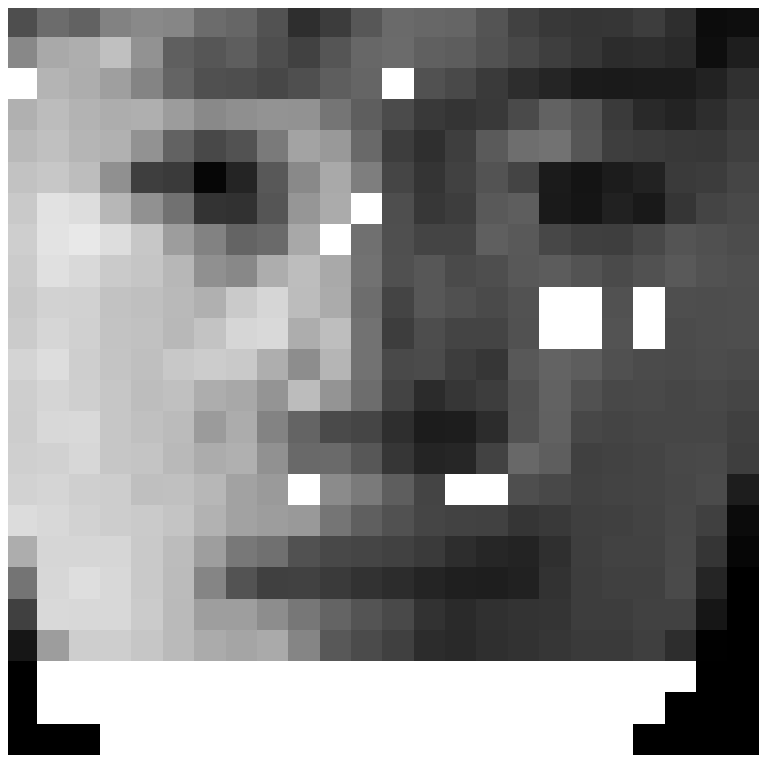}
\includegraphics[height=1.2cm]{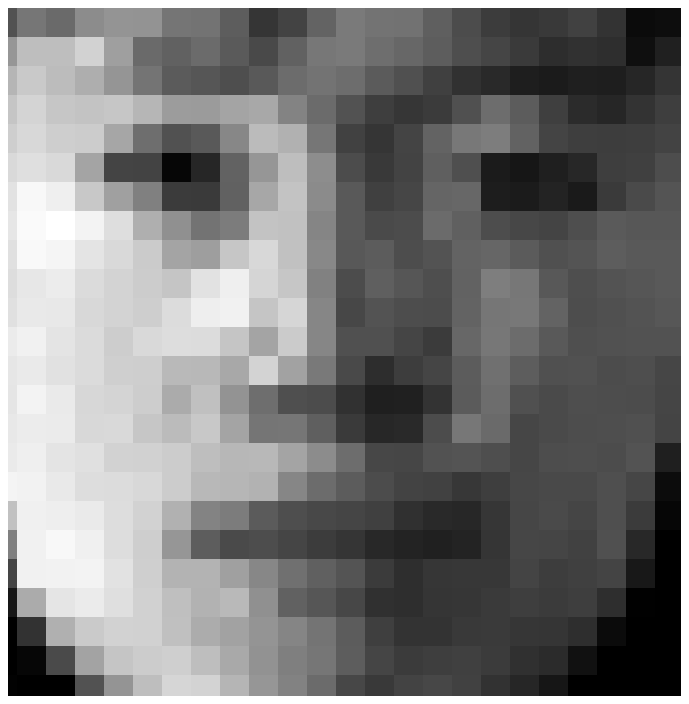}

\includegraphics[height=1.2cm]{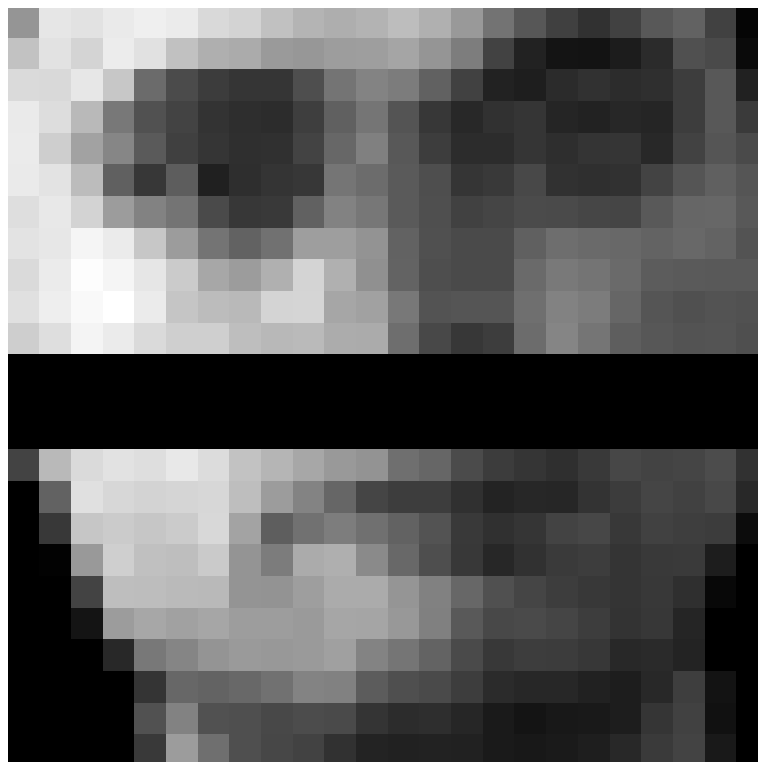}
\includegraphics[height=1.2cm]{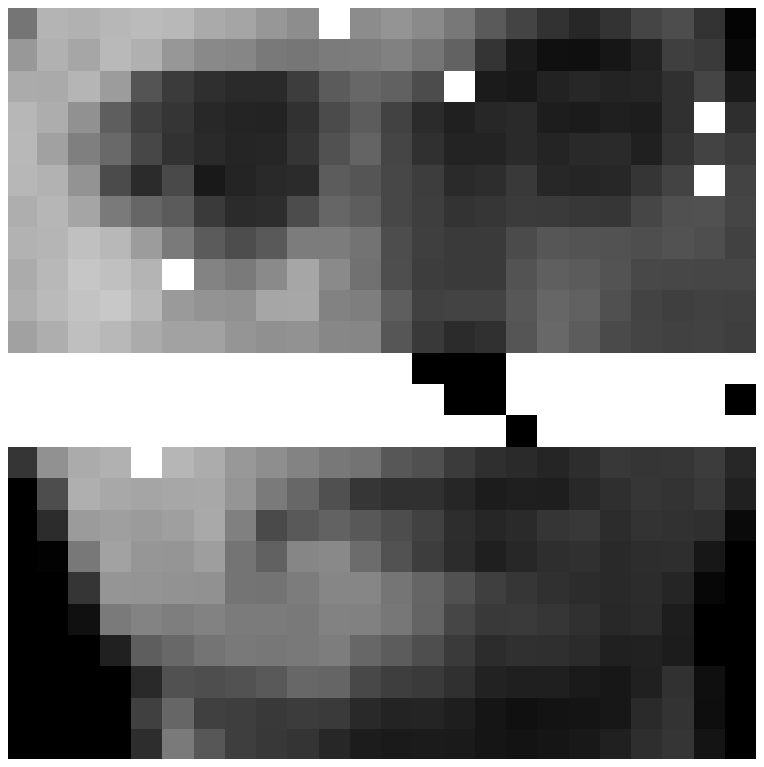}
\includegraphics[height=1.2cm]{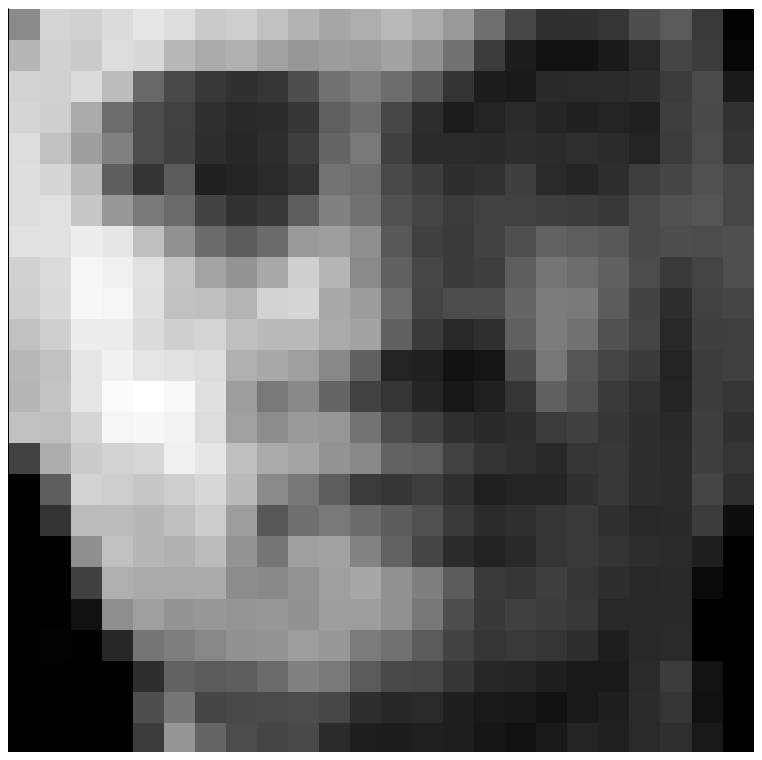}
\caption{This figure demonstrates the detection of the dead pixels, covering about 13\% of the input image. Each row shows a distinct example drawn from the CMU-PIE dataset. Outliers detected by our method are marked as white pixels in the centre column, followed by the reconstructed images in the last column.}
\label{Fig:PIE}
\end{figure}

As shown in the previous example, occlusion can significantly reduce face recognition performance, particularly in methods without outlier removal.  
Wright et al.\ \cite{Wright09} attempt to identify faces on the basis of
particular sub-images, such as the area around the eye or ear, etc.
Here, we  use the complete face and remove increasing portions of
the bottom half of the image, so that initially the neck is
obscured, followed by the chin and mouth, etc. The removal occurs
by setting the pixels to be black. Thus, the complete image is used
as a feature vector, and a subset of the elements is set to zero.
A second experiment is performed following the same procedure, but to the
central section of the face, thus initially obscuring the nose, then
the ears, eyes and mouth, etc.

In this experiment, we use the CMU-PIE dataset \cite{Sim03thecmu} which contains 68 subjects and a total of 41368 face images. Each person has their picture taken with 13 different poses, 43 different illumination conditions, and with 4 different expressions. In our experiment all of the face images are aligned and cropped, with 256 gray level per pixel \cite{He05laplacianscore}, and finally resized to $24 \times 24$ pixels. Here we use the subset containing images of pose C27 (a nearly front pose) and we  use the data from the first 20 subjects, each subject with 21 images. 
The first 15 images of each subject are used for training and the last 6 images for testing. The test images are preprocessed so that one part (bottom or middle) of faces (from $10\%$ to $40\%$ of pixels) are set to black. See Figure~\ref{Fig:PIE} for examples. 

The recognition rates of different methods when the bottom area of the image is occluded are reported in Table~\ref{Tab:PIE_bottom}. When occlusion area is small, all methods except MM-estimator obtain perfect 100\% recognition rates. When occlusion area increases to 20\% of the image size, accuracy for LRC  drops to 80\%, which is because the black pixels bias the linear regression estimate. The technique used in SRC mentioned above performs better than LRC when occlusion are present, achieving 98.3\% accuracy. However our method is able to achieve 100\% accuracy with that level of occlusion. We can see that CRC achieves a relatively good result (83.3\%) for 30\% occlusion, accuracies of other methods (including robust methods MM-estimator and LMS)  drop dramatically. In contrast, our method still achieves 100\% accuracy which demonstrates the robustness of our method against heavy occlusion. 
 The comparison of these methods for occlusion in the middle part of faces is shown in Figure~\ref{Tab:PIE_middle}. These results again show the robustness of our method against heavy occlusions. Almost all the methods show lower accuracy than in the former situation. Such a results leads to the conclusion that information from the middle part of a face (area around nose) is more discriminative than that form the bottom part (area around chin) for face recognition.


\begin{table}[h!]
\centering
\begin{tabular}{cccc}
\hline
\multirow{2}{*}{\textbf{Method}} & \multicolumn{3}{c}{\textbf{Percentage of image removed}} \\
 & 10\% & 20\% & 30\%\\
\hline
LRC & \textbf{100\%} & 80.0\% & 61.7\% \\
SRC & \textbf{100\%} & 98.3\% & 71.7\% \\
CRC & \textbf{100\%} & 96.7\% & 83.3\% \\
\hline
MM-estimator & 99.2\% & 41.7\% & 15.0\% \\
LMS & \textbf{100\%} & 77.5\% & 58.3\% \\
our method & \textbf{100\%} & \textbf{100\%} & \textbf{100\%} \\
\hline
\end{tabular}
\caption{Recognition accuracies  of various methods on the CMU-PIE dataset with dimension $24 \times 24$. $10\%$ to $30\%$ of pixels in the bottom area are replaced with black.}
\label{Tab:PIE_bottom}
\end{table}

\begin{table}[h!]
\centering
\begin{tabular}{cccc}

\hline
\multirow{2}{*}{\textbf{Method}} & \multicolumn{3}{c}{\textbf{Percentage of image removed}} \\
  & 10\% & 20\% & 30\% \\
\hline
LRC & \textbf{100\%} & 91.7\% & 35.0\% \\
SRC & \textbf{100\%} & 93.3\% & 65.0\%  \\
CRC & 98.3\% & 87.5\% & 53.3\% \\
\hline
MM-estimator & 92.5\% & 7.5\% & 0.8\% \\
LMS & \textbf{100\%} & 80.8\% & 22.5\% \\
our method & \textbf{100\%} & \textbf{99.2\%} & \textbf{90.0\%} \\
\hline
\end{tabular}
\caption{Recognition accuracies of various methods on the CMU-PIE dataset with dimension $24 \times 24$. $10\%$ to $30\%$ of pixels in the middle area are replaced with black.}
\label{Tab:PIE_middle}
\end{table}

\subsubsection{Efficiency} For the problem of identifying outliers in face images, we compare the computation efficiency using the AR face dataset, as described in Section~\ref{subsub:disguise}. We vary the feature dimension from 54 to the original 19800. Table~\ref{Tab:time_face} shows the execution time for both the proposed fast algorithm and the original method. We can see that the fast algorithm outperforms the original in all situations. With low dimensional features, below 4800, the fast algorithm is approximately 20 times faster than the original. When the feature dimension increases to 19800, the original algorithm needs about 1.43 hours while the fast algorithm costs only about 6 minutes.

\begin{table}[t]
\centering
\begin{tabular}{cccccc}
\hline
\multirow{2}{*}{\textbf{Method}} & \multicolumn{5}{c}{\textbf{Feature dimension}} \\
 & 54 & 300 & 1200 & 4800 & 19800\\
\hline
original & 2.051 & 9.894 & 48.371 & 396.323 & 5150.137\\
fast & 0.113 & 0.566 & 2.564 & 18.811 & 361.689\\
\hline
\end{tabular}
\caption{Computation time (in seconds) of the original and fast algorithms when applied to the AR face dataset.}
\label{Tab:time_face}
\end{table}

\comment{I don't think these experiments make any sense. They should be redone or removed. --- I did these experiments because last time when we submitted our TIP paper (also about robust fitting) the reviewer said we should show the efficiency of our method
on the face recognition problem. Our main contribution is we apply the robust fitting method on the face recognition problem and in a efficient way.}
\subsection{Robust iris recognition}
Iris recognition is a commonly used non-contact biometric measure used to automatically identify a person. Occlusions can also occur in iris data acquisition, especially in unconstrained conditions, caused by eyelids, eyelashes, segmentation errors, etc. In this section we test our method against segmentation errors, which can result in outliers from eyelids or eyelashes. 
 Specially we take the ND-IRIS-0405 dataset \cite{ND06}, which contains 64,980 iris images obtained from 356 subjects with a wide variety of distortions.  In our experiment, each iris image is segmented by detecting the pupil and iris boundaries using the open-source package of Masek and Kovesi \cite{MasekIris03}. 
80 subjects were selected and 10 images from each subject were chosen for training and 2 images for testing. 
To test outlier detection, segmentation errors and artificial occlusions were placed on the iris area, in a similar fashion as \cite{Iris2011}. A few example images and their detected iris and pupil boundaries are shown in Figure~\ref{Fig:Iris}. 
 The feature vector is obtained by warping the circular iris region into a rectangular block by sampling with a radial resolution 20 and angular resolution of 240 respectively. These blocks were then are then resized to $10 \times 60$.  
 For our method, 10\% of pixels are detected and removed when test images are with only segmentation errors, and the corresponding additional number of pixels are removed for artificial occlusions. 

The recognition results are summarized in Table \ref{Tab:Iris}.  SRC used in \cite{Iris2011} for iris recognition and LRC  are compared with our method.
We can clearly see that the proposed method achieved the best results
with all feature dimensions. Specifically, our method achieves 96.3\%
accuracy when iris images are with only segmentation errors while
accuracy for LRC is 89.5\%. SRC performs well (95.6\%) for this task.
However when 10\% additional occlusions occur in the  test images,
performances for LRC and SRC drop dramatically to 43.8\% and 61.3\%
respectively, while our method still achieves the same result 96.3\% as before. When occlusions increase to 20\%, our method still obtains a high accuracy 95\% which is higher than those of LRC and SRC by 74.4\% and 43.7\% respectively.

\subsubsection{Efficiency}

Table \ref{Tab:time_iris} shows the computation time comparison of different methods on the iris recognition problem. Consistent with the former results, the proposed algorithm is much more efficient than the original algorithm.

\begin{figure}
\centering
\includegraphics[height=1.2cm]{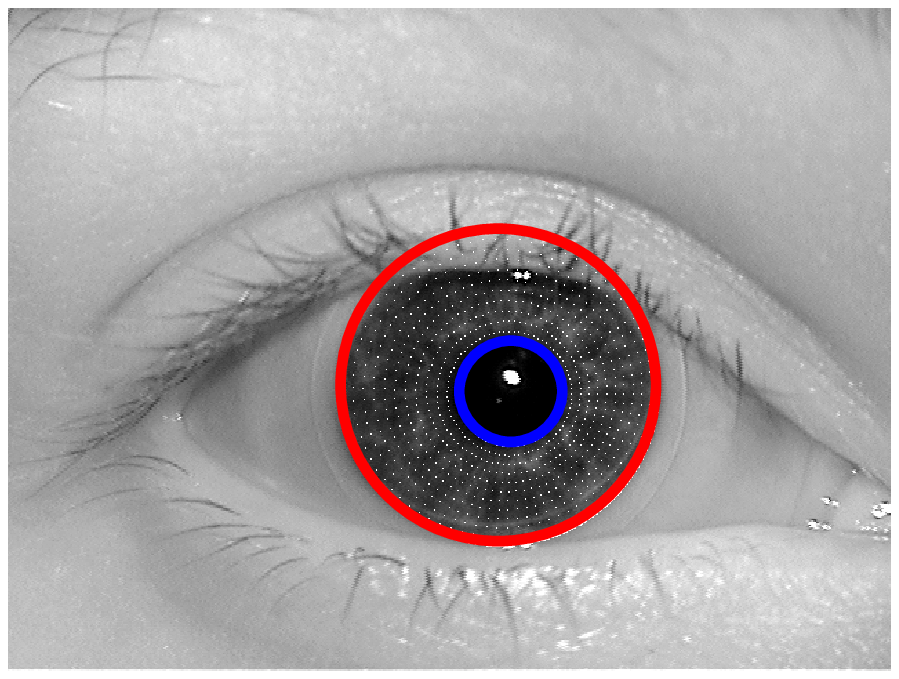}
\includegraphics[height=1.2cm]{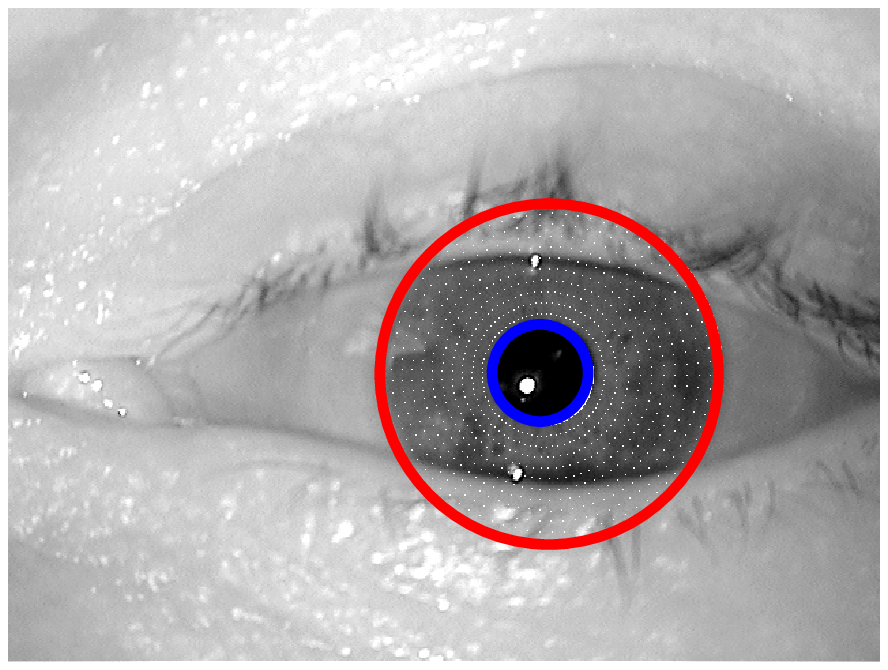}
\includegraphics[height=1.2cm]{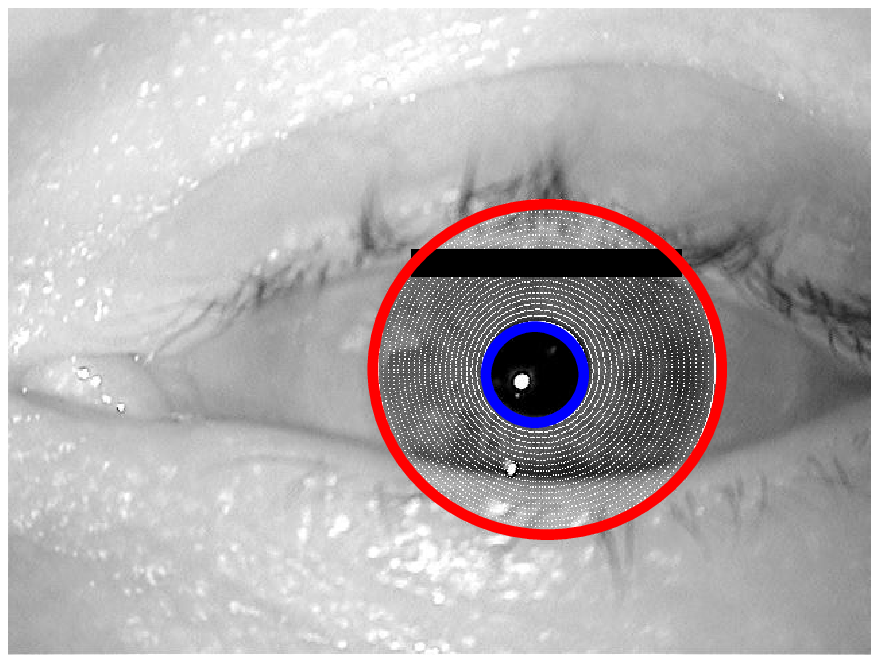}\\
\caption{Three example images from the ND-IRIS-0405 dataset. Features are extracted from the iris area which is between the detected iris and pupil boundaries as shown in red and blue circles respectively. The first two iris images suffer from increasing segmentation errors while the third one suffers from both segmentation error and artificial occlusion. }
\label{Fig:Iris}
\end{figure}


\begin{table}[t]
\centering
\begin{tabular}{ccccc}
\hline
\multirow{2}{*}{\textbf{Method}} & \multicolumn{3}{c}{\textbf{Percentage of artificial occlusion}} \\
 & 0\% & 10\% & 20\% \\
\hline
LRC        & 89.5\% & 43.8\% & 20.6\%  \\
\hline
SRC        & 95.6\% & 61.3\% & 51.3\% \\
\hline
our method & \textbf{96.3\%} & \textbf{96.3\%} & \textbf{95\%}\\
\hline
\end{tabular}
\caption{Classification accuracies  on the ND iris dataset. Note that percentage of artificial occlusion 0\% means there are only segmentation errors.}
\label{Tab:Iris}
\end{table}

\begin{table}[h]
\centering
\begin{tabular}{ccccc}
\hline
\multirow{2}{*}{\textbf{Method}} & \multicolumn{4}{c}{\textbf{Image resolution}} \\
 &  $5\times 20$&  $5\times 60$  &  $10 \times 120$ & $20 \times 240$ \\
\hline
original & 2.400 & 6.692 & 31.942 & 255.182\\
fast & 0.131 & 0.385 &1.941 & 14.906\\
\hline
\end{tabular}
\caption{Computation time (in seconds) comparison of the original and fast algorithms on the Iris data set with different feature resolutions (shown in the first row). }
\label{Tab:time_iris}
\end{table}

\section{Discussion}
The main drawback of our method is that one have to first estimate the
outlier percentage empirically as done by many other robust regression
methods. 
Actually, to our knowledge, for almost all the {\em outlier removal}
methods, one has to pre-set the outlier percentage or some other
parameters such as a residual threshold.  This is in contrast with
those robust regression methods using a robust loss such as the Huber
function or even nonconvex loss. These methods do not need to specify
the outlier percentage.

One may concern how the proposed algorithm will perform with
an under or over estimated $p$. Taking the AR dataset for example, we
evaluate our method by varying $p$ from 25\% to 45\%. From Table
\ref{Tab:vary_percentage}, we can see that the proposed method is not
very sensitive to the pre-estimated outlier percentage when $p$ is
over 30\%. We also observe that our method becomes more stable when
the image resolution is higher. This is mainly because, as mentioned
before, visual recognition problems generally supply large amount of
pixels by high dimensional images and consequently it is more crucial
to reject as many outliers as possible than to keep all inliers.

Different from our approach, there exist many robust estimators which
do not need to specify the outliers number, such as MM-estimator
\cite{yohai1987high}, LMS \cite{LMS1984} and DPM
\cite{park2012robust}. These methods can also be applied to visual
recognition problems as we have shown  in Section \ref{SEC:Exp}.
However, the difference is that our method can directly identify the
outliers, which can help compute more reliable residuals for
classification as shown in Section \ref{SEC:Exp}.  Of course, for these
methods, observations can be detected as outliers when the
corresponding standardized residuals exceed the cutoff point, which
also has to be determined {\em a priori} though.

\begin{table}[h]
\centering
\begin{tabular}{cccccc}
\hline
\multirow{2}{*}{\textbf{Dimension}} & \multicolumn{5}{c}{\textbf{Percentage of removed pixels}} \\
 & 25\% & 30\% & 35\% & 40\% & 45\%\\
\hline
$13 \times 10$ & 82\% & 85\% & 92.5\% & 95\% & 94.5\%\\
$20 \times 15$ & 98\% & 99.5\% & 100\% & 100\% & 99.5\%\%\\
\hline
\end{tabular}
\caption{Recognition accuracies  on the AR dataset  with different percentage of outliers removed by our method. The feature dimension is set to $13 \times 10$ and $13 \times 10$.}
\label{Tab:vary_percentage}
\end{table}

\section{Conclusion}
\label{SEC:conclusion}
In this work, we have proposed an efficient method for minimizing the
$L_\infty$ norm based robust least squares fitting,
and hence for iteratively removing outliers. The efficiency of the method allows it to be
applied to visual recognition problems which would normally be too large for such an approach.
The method takes advantage of the nature of the $L_\infty$ norm to break the main problem
into more manageable sub-problems, which can then be solved via standard, efficient, techniques.

The efficiency of the technique and the benefits that outlier removal can bring to visual
recognition problems were highlighted in the experiments, with the computational efficiency
and accuracy of the resultant recognition process easily beating all other tested methods.

Like many other robust fitting methods, the proposed
method needs a parameter: the number of  outliers to be removed.
One may heuristically determine this value. Although it is not very sensitive for the visual recognition problems, in the future, we
plan to investigate how to automatically estimate the outlier rate in the noisy data.

\section*{Acknowledgements} 

    This work was in part supported by ARC Future Fellowship FT120100969.
    F. Shen's contribution was made when he was visiting The
    University of Adelaide. 

    All correspondence should be addressed to C. Shen
    (\url{chunhua.shen@adelaide.edu.au}).

\bibliographystyle{elsarticle-num}
\bibliography{CSRef}

\end{document}